\documentclass[11pt, a4paper]{article}

\usepackage[utf8]{inputenc}    
\usepackage[T1]{fontenc}       
\usepackage{amsmath}           
\usepackage{amssymb}           
\usepackage{amsfonts}          
\usepackage{graphicx}          
\usepackage{amsthm}            
\usepackage{hyperref}          
\usepackage[left=2.5cm, right=2.5cm, top=2.5cm, bottom=2.5cm]{geometry} 
\usepackage{setspace}          
\usepackage{abstract}          
\usepackage{enumitem}          
\usepackage{xurl}              
\usepackage{mathtools}         
\usepackage{booktabs}          
\usepackage{array}             
\usepackage{subcaption}        
\usepackage{multirow}          
\usepackage{algorithm}         
\usepackage{xcolor}            

\newtheorem{definition}{Definition}[section] 

\newtheorem{theorem}{Theorem}[section]

\theoremstyle{definition} 
\newtheorem{example}{Example}[section]
\newtheorem{remark}{Remark}[section]

\newcommand{\semigroup}{\mathcal{S}}

\newcommand{\responsevar}{Y} 
\newcommand{\rewards}{\mathbb{R}}
\newcommand{\initialpatients}{P_0} 
\newcommand{\equivalence}{\sim}
\newcommand{\binaryvector}{\mathbf{b}} 
\newcommand{\atomicrules}{\mathcal{A}} 
\newcommand{\powerset}{\mathfrak{P}} 
\newcommand{\objectivefun}{f} 

\title{Algebraic Structure Discovery for Real World Combinatorial Optimisation Problems:
	A General Framework from Abstract Algebra to Quotient Space Learning}

\author{
	Min Sun\thanks{\texttt{min.sun.ms1@roche.com}} \and
	Federica Storti\thanks{\texttt{federica.storti@roche.com}} \and
	Valentina Martino\thanks{\texttt{valentina.martino@roche.com}} \and
	Miguel Gonzalez-Andrades\thanks{\texttt{miguel.gonzalez\_andrades@roche.com}} \and 
	Tony Kam-Thong\thanks{\texttt{tony.kam-thong@roche.com}}
}

\date{
	\small F. Hoffmann-La Roche AG, Roche Pharma Research and Early Development\\
	\today
}

\begin{document}
	
	\maketitle
	
	\begin{abstract}
		Many combinatorial optimisation problems hide algebraic structures that, once exposed, shrink the search space and improve the chance of finding the global optimal solution. We present a general framework that (i) identifies algebraic structure, (ii) formalises operations, (iii) constructs quotient spaces that collapse redundant representations, and (iv) optimises directly over these reduced spaces. Across a broad family of rule-combination tasks (e.g., patient subgroup discovery and rule-based molecular screening), conjunctive rules form a monoid. Via a characteristic-vector encoding, we prove an isomorphism to the Boolean hypercube $\{0,1\}^n$ with bitwise OR, so logical AND in rules becomes bitwise OR in the encoding. This yields a principled quotient-space formulation that groups functionally equivalent rules and guides structure-aware search. On real clinical data and synthetic benchmarks, quotient-space-aware genetic algorithms recover the global optimum in 48–77\% of runs versus 35–37\% for standard approaches, while maintaining diversity across equivalence classes. These results show that exposing and exploiting algebraic structure offers a simple, general route to more efficient combinatorial optimisation.
	\end{abstract}
	
	\doublespace 
	
	\section{Introduction}
	
	Combinatorial optimisation problems arise across diverse domains — from drug discovery, clinical research to logistics — where researchers must find optimal combinations of discrete components under complex constraints. Consider the traveling salesman problem seeking optimal city visit sequences, patient subgroup finding over categorical and continuous clinical features, or placing samples in a sequence machine such that the adjacent samples should be as different as possible for better randomisation. When these problems are treated as unstructured search spaces, standard approaches often suffer from exponential computational complexity, poor convergence to global optima, and inability to exploit underlying mathematical regularities that could dramatically reduce search effort. We propose a systematic framework for discovering and leveraging algebraic structure in combinatorial problems to achieve dramatically improved optimisation efficiency.
	
	\subsection{The General Framework}
	Our approach consists of four sequential steps:
	
	\begin{enumerate}
		\item \textbf{Structural Analysis:} Examine the real-world combinatorial problem to identify potential algebraic properties of its components and operations.
		\item \textbf{Algebraic Formalisation:} Map the problem structure to concepts from abstract algebra structures (groups, monoids, semigroups) to establishing formal operations and their properties.
		\item \textbf{Quotient Space Construction:} Identify equivalence relations that reveal redundancy in the search space, then construct quotient spaces that eliminate this redundancy while preserving optimisation objectives.
		\item \textbf{Structure-Aware Optimisation:} Design algorithms that explicitly exploit the algebraic structure and operate efficiently over the reduced quotient spaces.
	\end{enumerate}
	
	While not all combinatorial problems admit useful algebraic structure, we demonstrate that this framework can yield substantial improvements for a meaningful class of problems where such structure exists.
	
	\textbf{Bridging Theory and Practice.} Abstract algebra encompassing groups, rings, fields, and other algebraic structures, has traditionally been perceived as belonging to pure mathematics and theoretical physics, with limited direct application to everyday data analysis problems. Most data scientists and applied researchers encounter algebraic concepts primarily in linear algebra \cite{deisenroth2020mathematics}, but rarely consider more sophisticated algebraic structures as practical tools for combinatorial optimisation. This perception has created a substantial gap: while abstract algebra provides powerful conceptual frameworks for understanding structure and equivalence, these tools remain largely under used in real world combinatorial problems where they could offer significant computational advantages. Our work demonstrates that this gap represents a missed opportunity — many real world combinatorial problems naturally exhibit algebraic structure that, when properly identified and exploited, can transform intractable search spaces into manageable, mathematically principled optimisation challenges.
	
	\subsection{Motivating Applications}
	We demonstrate our framework on two distinct yet structurally similar combinatorial optimisation challenges:
	
	\textbf{Patient Subgroup Discovery (detailed with benchmark):} In personalised medicine, researchers seek to identify patient subgroups characterised by specific clinical characteristics combinations. Traditional approaches often see this as a classification problem \cite{Zhang2018subgroup}, but the issue is that not all such problems have fixed label. 
	
	Imagine the case where researchers would like to find a patient subgroup whose some protein level needs to be much elevated comparing to healthy volunteers based on the value of some clinical tests. For each patient, the label of whether he/she belongs to this subgroup is not fixed (the label may change if the set of clinical tests changes). This is contrary to classification problems where the label of each datapoint is always fixed. On the other hand this is also not a clustering problem as clustering doesn't even have labels. In fact when labels are not fixed, this problem is much closer to playing a video game (e.g. the well known Super Mario) once we analyse it from algebraic perspective.
	
	In Super Mario, the player combines atomic actions \{Left, Right, Up, Down, Jump\} through sequential composition to navigate towards a goal state, where different sequences of the same moves can lead to equivalent outcomes or entirely different results depending on context. Similarly, in patient subgroup discovery, researchers combine atomic clinical criteria \{clinical test 1 > x, clinical test 2 < y ...\} through logical conjunction to define patient subgroup. Both problems exhibit: (1) \emph{compositional structure} where atomic elements are combined via composition (2) \emph{equivalence classes} where multiple distinct action sequences can achieve the same functional result. This algebraic perspective transforms what appears to be heuristic search into principled mathematical optimisation over structured spaces.
	
	\textbf{Rule-Based Molecular Screening (no benchmark, only the concept):} In drug discovery, researchers apply combinations of Boolean filter predicates to eliminate compounds with poor drug-likeness, following the 'fail early, fail cheaply' principle \cite{Oprea2002,VSbestPractice2021,PhamThe2020,Sukhachev2025}. Here researchers could be interested in finding a subset of molecules which predicted to have desired properites, e.g. high binding affinity or other ADME (absorption, distribution, metabolism, and excretion) metrics. The optimal combination of filters represents a complex combinatorial optimisation problem.
	
	Both problems involve combining Boolean predicates through conjunction, suggesting they may share underlying algebraic structure. Our analysis reveals that both naturally exhibit \textbf{monoid structure}, enabling systematic discovery of \textbf{equivalence classes} of functionally identical rule combinations. This transforms intractable combinatorial search into structured optimisation over well-defined quotient spaces.
	
	In simpler terms: imagine you have a toolbox of simple yes/no rules (like "patient age > 65" or "molecule weight < 500"). The key insight is that when you combine these rules using "AND" logic, you are actually following mathematical patterns that can be exploited. Many different combinations of rules will identify exactly the same patients or molecules - these are "equivalent" even though they look different on paper. Instead of testing millions of random combinations, our approach groups together all the equivalent rule combinations and tests only one representative from each group. This dramatically reduces the search space while guaranteeing we do not miss the optimal solution - like having a smart filing system that eliminates duplicate folders before you start searching. On the other hand, by keeping one individual from each group (equivalence class), we ensure that our search explores all the different "types" of solutions, preventing us from getting stuck in one corner of the search space. This balance between reducing redundancy and maintaining diversity is what makes our method both efficient and effective.
	
	\subsection{Key Contributions} \label{overall_framework}
	\begin{enumerate}
		\item A general framework for discovering and exploiting algebraic structure in combinatorial optimisation problems
		\item Formal proof that conjunctive rule problems exhibit monoid structure with natural quotient space constructions
		\item Demonstration that the framework applies to diverse domains (with clinical research as the main case study)
		\item Empirical validation (with real clinical data) showing 48-77\% global optimum achievement versus 35-37\% for standard approaches
		\item Open-source implementation enabling adoption across disciplines
	\end{enumerate}

	\section{Case Study: Patient Subgroup Discovery}
	
	We demonstrate our framework using patient subgroup discovery as the primary case study, then show how identical algebraic structure also emerges in other real world scenarios such as molecular screening.
	
	\subsection{Step 1: Structural Analysis of the Real-World Problem}
	Let $\initialpatients$ denote the finite set of all available patients in a given cohort. For each patient $p \in \initialpatients$, we possess a vector of clinical measurements $C(p) \in \rewards^k$ and a continuous biomarker measurement $\responsevar(p) \in \rewards$. Our objective is to identify a patient subgroup $P^* \subseteq \initialpatients$ defined by a logical rule $r^*$ (on a set of clinical measurements) such that the mean biomarker fold change of $P^*$ relative to a healthy volunteer (HV) cohort is maximised.
	
	Researchers typically combine clinical criteria using logical conjunction. For example: patients with Ocular Surface Disease Index (OSDI)
	score > 12 AND Tear Break-Up Time (TBUT) < 10 seconds AND presence of meibomian gland dysfunction (these are typical assessments of the ocular surface and tear film health). This suggests that the fundamental operation is the conjunction ($\land$) of Boolean-valued functions. In practice, the list of criteria is often a mix of continuous and categorical variables, which makes gradient based optimisation unsuitable.
	
	From a mathematical perspective, the combinatorial explosion occurs because researchers must explore all possible combinations of available clinical criteria.

	\subsection{Step 2: Algebraic Formalisation}
	
	\begin{definition}[Atomic Clinical Rules]
		Let $\atomicrules = \{a_1, a_2, \ldots, a_n\}$ be a finite set of $n$ distinct \textbf{atomic clinical rules}. Each $a_i$ is a boolean predicate defined on a patient's clinical measurements $C(p)$ (e.g., $a_1(p) = (\text{OSDI}(p) > 12)$). Applying $a_i$ to a set of patients $P$ yields the subset $a_i(P) = \{p \in P \mid p \text{ is selected by } a_i\}$.
	\end{definition}
	
	\begin{definition}[Monoid of Conjunctive Rules]
		Let $\semigroup$ denote the set of all \textbf{composite rules} formed by conjunctions of atomic rules from $\atomicrules$, including the \textbf{identity rule} $\varepsilon$ (empty conjunction) that applies no filtering.
		
		The binary operation $\cdot$ on $\semigroup$ represents rule combination via logical conjunction. For any $r_1, r_2 \in \semigroup$, the composite rule $r_1 \cdot r_2 \equiv r_1 \land r_2$ represents the conjunction of both rules.
		
		This operation is associative and admits an identity element $\varepsilon$, forming the monoid $(\semigroup, \land, \varepsilon)$.
	\end{definition}
	
	\begin{definition}[Monoid Action on Patient Populations]
		The monoid $\semigroup$ acts on patient populations $\initialpatients$ via the filtering operation $r(P) = \{p \in P \mid p \text{ is selected by } r\}$, satisfying the action axioms: (1) $\varepsilon(P) = P$ (identity preservation), and (2) $(r_1 \cdot r_2)(P) = r_1(r_2(P))$ (associativity).
	\end{definition}

	\begin{remark}[Practical Semigroup Action]
		While the mathematical structure forms a complete monoid, in practical optimisation we exclude the identity element $\varepsilon$ (which applies no filtering and returns the entire patient population). This restriction makes our working space effectively a semigroup $\semigroup \setminus \{\varepsilon\}$, where all considered rules perform meaningful patient filtering. Thus, we primarily work with a \textbf{semigroup action} on patient populations, though the theoretical foundation remains a monoid.
	\end{remark}

	\begin{remark}[Super Mario as an algebraic structure]
		 The monoid structure reveals that combining rules is not merely an ad-hoc process, but follows precise algebraic laws. This mathematical perspective allows systematic analysis of the search space. This is where we see the patient stratification problem similar to playing \emph{Super Mario}, where movements $\{\text{Left, Right, Up, Down, Jump}\}$ can be combined through the binary operation 'composition of moves', acting on Mario's current position. Here the algebra is clear: actions compose sequentially and define a state transition system. However, the semantics of moves may introduce context-dependence (e.g., pressing 'Down' before 'Left' may send Mario underground via a tube, whereas reversing the order does not). This illustrates that processes which at first appear heuristic can often be understood in algebraic terms once the underlying structure and equivalences are identified.
	\end{remark}
	
	\subsection{Step 3: Quotient Space Construction}
	
	The critical insight for optimisation efficiency comes from recognising that many distinct clinical rules yield functionally equivalent outcomes.
	
	\subsubsection{Computational Representation via Isomorphism}
	Before constructing quotient spaces, we establish a computationally convenient representation:
	
	\begin{definition}[Binary Vector Representation of Rules]
		Let $V = \{0,1\}^n$ be the set of all $n$-dimensional binary vectors. We define a mapping $\phi: \semigroup \to V$ for each composite rule $r \in \semigroup$. Given that $r$ is characterised by its set of atomic rules $S_r \subseteq \atomicrules$, we define $\phi(r) = \binaryvector = (b_1, b_2, \ldots, b_n) \in V$, where $b_i = 1$ if $a_i \in S_r$ (i.e., $a_i$ is a conjunct in $r$), and $b_i = 0$ otherwise.
	\end{definition}
	
	\begin{remark}[AND in rules becomes OR in encoding]
		Rule composition uses logical conjunction ($\land$), which corresponds to set union of atomic predicates: $S_{r_1 \land r_2} = S_{r_1} \cup S_{r_2}$. Under the 0-1 encoding $\phi$, this union becomes bitwise OR ($\lor$) on the Boolean hypercube: $\phi(r_1 \land r_2) = \phi(r_1) \lor \phi(r_2)$. We emphasise 'Boolean hypercube with bitwise OR' to make this mapping explicit.
	\end{remark}

	\begin{example}[Binary Vector Representation of Rules]
		Consider the set of atomic rules $\atomicrules = \{\text{Dry eye type: aqueous deficient, Dry eye type: evaporative, Dry eye severity: mild}\}$ and a composite rule $r$ characterized by $S_r = \{\text{Dry eye type: aqueous deficient AND Dry eye severity: mild}\}$. The resulting binary vector is $\phi(r) = \binaryvector = (1, 0, 1)$.
	\end{example}
	
	\begin{theorem}[Isomorphism to the Boolean hypercube with bitwise OR]
		The semigroup $(\semigroup, \land)$ is isomorphic to $(V, \lor)$ on the Boolean hypercube, where $V = \{0,1\}^n$ and $\lor$ denotes element-wise (bitwise) logical OR. 
	\end{theorem}
	
	\begin{proof}
		To establish isomorphism, we must demonstrate that the mapping $\phi: \semigroup \to V$ is a bijection and that it preserves the semigroup operation.
		
		\begin{enumerate}
			\item \textbf{Bijection:}
			\begin{itemize}
				\item \textbf{Injectivity:} Assume $\phi(r_1) = \phi(r_2)$ for $r_1, r_2 \in \semigroup$. This implies that their corresponding binary vectors are identical, i.e., $b_{1,i} = b_{2,i}$ for all $i=1,\ldots,n$. By definition of $\phi$, this signifies that $S_{r_1} = S_{r_2}$, and thus $r_1 = r_2$. Hence, $\phi$ is injective.
				\item \textbf{Surjectivity:} For any binary vector $\binaryvector = (b_1, \ldots, b_n) \in V$, we can construct a composite rule $r \in \semigroup$ by defining its set of atomic rules as $S_r = \{a_i \in \atomicrules \mid b_i = 1\}$. This rule $r$ is a conjunction of precisely those atomic rules corresponding to the '1's in $\binaryvector$. By construction, $\phi(r) = \binaryvector$. Hence, $\phi$ is surjective.
			\end{itemize}
			Since $\phi$ is both injective and surjective, it is a bijection.
			
			\item \textbf{Operation Preservation:} We need to show that $\phi(r_1 \land r_2) = \phi(r_1) \lor \phi(r_2)$, where $\land$ on the left side is the semigroup operation on rules (logical conjunction, corresponding to union of atomic rule sets) and $\lor$ on the right side is bitwise logical OR.
			Let $r_1$ be characterised by $S_{r_1}$ and $r_2$ by $S_{r_2}$.
			The composite rule $r_1 \land r_2$ is characterised by the set of atomic rules $S_{r_1 \land r_2} = S_{r_1} \cup S_{r_2}$.
			Now consider the $i$-th component of the binary vector $\phi(r_1 \land r_2)$:
			$$ (\phi(r_1 \land r_2))_i = 1 \iff a_i \in S_{r_1 \land r_2} $$
			This is true if and only if $a_i \in S_{r_1} \cup S_{r_2}$, which means $a_i \in S_{r_1}$ or $a_i \in S_{r_2}$.
			By definition of $\phi$, this is equivalent to $(\phi(r_1))_i = 1$ or $(\phi(r_2))_i = 1$.
			Therefore, $(\phi(r_1 \land r_2))_i = (\phi(r_1))_i \lor (\phi(r_2))_i$ for all $i=1,\ldots,n$.
			This confirms that $\phi(r_1 \land r_2) = \phi(r_1) \lor \phi(r_2)$, preserving the operation.
		\end{enumerate}
		Since $\phi$ is a bijection and preserves the operation, it is an isomorphism. Consequently, we can equivalently perform our search and modelling on the Boolean hypercube $V$ utilising bitwise OR as the composition operator, which aligns with the structure implicitly handled by the Hamming distance metric.
	\end{proof}

	\begin{remark}[Computational significance on the Boolean hypercube]
		The isomorphism $\phi: (\semigroup, \land) \cong (V, \lor)$ provides crucial computational advantages for solving our combinatorial optimisation problem on the Boolean hypercube with bitwise OR:
		
		\textbf{Tractable Representation:} Rather than working directly with abstract logical rules, we can represent each rule as a simple binary vector where each bit indicates whether a particular atomic criterion is included. This transforms complex logical operations into straightforward bitwise operations that are natively supported by modern computing architectures.
		
		\textbf{Efficient Distance Computation:} The isomorphism enables natural use of Hamming distance $d_H(\binaryvector_1, \binaryvector_2) = \sum_{i=1}^n |b_{1,i} - b_{2,i}|$ to measure similarity between rules. Rules that differ in few atomic criteria have small Hamming distance, whilst those using completely different criteria are far apart. This geometric structure facilitates clustering algorithms like DBSCAN for discovering equivalence classes.
		
		\textbf{Structure-Preserving Operations:} Genetic algorithm operators (crossover, mutation) can operate directly on binary vectors whilst respecting the underlying algebraic structure. Single-point crossover naturally corresponds to combining rule subsets, and bit-flip mutation corresponds to adding or removing atomic criteria from rules.
		
		In essence, the isomorphism bridges the semantic gap between our problem domain (logical rules for patient stratification) and computational machinery (binary vector arithmetic), enabling efficient implementation whilst preserving all essential mathematical structure.
	\end{remark}
	
	\begin{example}
		Consider atomic rules derived from categorical variables: Dry Eye (DED) severity $\in \{\text{healthy, mild, moderate, severe}\}$, Gender $\in \{\text{male, female}\}$, and Meibomian Gland Dysfunction (MGD) $\in \{\text{absent, present}\}$. This yields atomic rules:
		\begin{align*}
			\atomicrules = \{&a_1: \text{DED = healthy}, \quad a_2: \text{DED = mild}, \quad a_3: \text{DED = moderate}, \quad a_4: \text{DED = severe}, \nonumber \\
			&a_5: \text{Gender = male}, \quad a_6: \text{Gender = female}, \nonumber \\
			&a_7: \text{MGD absent}, \quad a_8: \text{MGD present}\} \nonumber
		\end{align*}
		The composite rule $r_1 = a_3 \land a_8$ (moderate dry eye with MGD) and rule $r_2 = a_6$ (female patients) map to binary vectors:
		\begin{align*}
			\phi(r_1) &= (0,0,1,0,0,0,0,1) \\
			\phi(r_2) &= (0,0,0,0,0,1,0,0)
		\end{align*}
		Under the isomorphism, combining rules corresponds to bitwise OR:
		\begin{align*}
			\phi(r_1 \land r_2) &= \phi(r_1) \lor \phi(r_2) \\
			&= (0,0,1,0,0,0,0,1) \lor (0,0,0,0,0,1,0,0) \\
			&= (0,0,1,0,0,1,0,1)
		\end{align*}
		representing female patients with moderate dry eye and MGD. The Hamming distance $d_H(\phi(r_1), \phi(r_2)) = 3$ indicates these rules differ in three atomic criteria, whilst genetic crossover might naturally combine or separate these categorical conditions.
	\end{example}

	\subsection{Equivalence Classes and Quotient Spaces: The Core of Patient Stratification}
	The fundamental insight of our algebraic approach lies in recognising that many distinct clinical rules yield functionally equivalent outcomes. This leads naturally to the concept of \textbf{equivalence classes} - collections of rules that produce similar clinical results. By identifying and exploiting these equivalence classes, we can dramatically reduce the complexity of the optimisation problem.
	
	\begin{definition}[Objective Function on Rules]
		For each rule $r \in \semigroup$, its application to the initial patient cohort $\initialpatients$ yields a subset of patients $r(\initialpatients)$. Denote the powerset of $P_0$ by $\powerset(P_0)$,  we define the objective function for a rule as $\objectivefun(r) \coloneqq R(r(\initialpatients))$, where $R: \powerset(\initialpatients) \to \rewards$ is the reward function quantifying the biomarker fold change (or other relevant metric) for a given subgroup vs healthy volunteers.
	\end{definition}
	
	\begin{definition}[Equivalence Relation on Rules]
		We define an \textbf{equivalence relation} $\equivalence$ on the semigroup of rules $\semigroup$ based on exact equality of objective function values. Two rules $r_1, r_2 \in \semigroup$ are considered \textbf{equivalent} ($r_1 \equivalence r_2$) if:
		$$r_1 \equivalence r_2 \quad \iff \quad \objectivefun(r_1) = \objectivefun(r_2)$$
		This defines a proper equivalence relation that partitions the rule space into equivalence classes of rules with identical objective function values.
	\end{definition}
	
	\begin{remark}[Practical Approximate Equivalence]\label{rem:practical-equivalence}
		While the theoretical equivalence relation requires exact equality of objective function values, in practice, e.g. our case study, a set of rules which gives 10 times higher protein level vs healthy volunteer makes no difference than another set of rules with 10.1 fold change to clinicians or clinical scientists. In this case they would instead value more on the feasibilty of recruitment of such subgroup of patients, not the absolute fold change. Therefore in implementation, we employ an \textbf{approximate equivalence} relation using a small tolerance $\epsilon > 0$:
		$$r_1 \approx r_2 \quad \text{if} \quad |\objectivefun(r_1) - \objectivefun(r_2)| < \epsilon$$
		This approximate relation, while not mathematically transitive, provides a practical clustering criterion for identifying functionally similar rules during optimization. The choice of $\epsilon$ represents a trade-off between mathematical precision and computational efficiency. For \emph{clinical applications, minimal clinically important difference} could be a good choice of $\epsilon$.
				
		Here the exact equivalence induces the true quotient whereas $\epsilon$-proximity is a practical surrogate for exploration and analysis.
	\end{remark}	
	
	\begin{remark}[Two exact equivalences and their roles]
		We distinguish two notion of equivalences:
		\begin{itemize}
			\item \emph{Functional equivalence:} $r_1 \equivalence r_2$ iff $\objectivefun(r_1)=\objectivefun(r_2)$ — this is our paper's canonical exact equivalence for quotienting the search space
			\item \emph{Subset equivalence:} $r_1 = r_2$ iff $r_1(\initialpatients)=r_2(\initialpatients)$, the subgroup of patients are identical under rule $r_1$ and $r_2$
		\end{itemize}
		In our case study the objective depends only on the induced subset, by definition subset equivalence implies functional equivalence.
	\end{remark}
	
	\begin{definition}[Quotient Space]
		For a set $X$ and an equivalence relation $\sim$ on $X$, the \textbf{quotient space} $X/\sim$ is the set of all equivalence classes of $X$ under $\sim$. Each element in $X/\sim$ is an equivalence class $[x] = \{y \in X \mid y \sim x\}$.
	\end{definition}
	
	\begin{remark}[Relationship between Equivalence Classes and Quotient Space]
		Our primary interest lies in grouping rules based on the similarity of their induced outcomes. The empirically learned equivalence classes defined by $\epsilon$-proximity in objective function values form a partition of the rule space $\semigroup$. This partition induces a "quotient space" $\semigroup/\equivalence$, where each element represents a set of rules that are functionally equivalent with respect to our optimisation objective. This allows us to conceptually and practically optimise over this reduced, effective search space.
	\end{remark}
	
	The existence of such equivalence classes implies redundancy in the search space. By identifying these, we can focus our computational efforts on evaluating only a representative from each class, significantly accelerating the search for optimal rule sets.

	\subsection{Step 4: Structure-Aware Optimisation Algorithms}
	
	Having established the general framework, we now complete our methodology by showing how to design optimisation algorithms that explicitly exploit algebraic structure and quotient spaces.
	
	\subsubsection{Key Algorithmic Principles}
	
	Structure-aware optimisation algorithms must incorporate several key principles:
	
	\begin{enumerate}
		\item \textbf{Equivalence Class Integration:} Algorithms must actively identify and exploit equivalence classes during the search process
		\item \textbf{Quotient Space Navigation:} Search procedures should operate over representatives of equivalence classes rather than the full search space
		\item \textbf{Structure-Preserving Operations:} Genetic operators and acquisition functions must respect the underlying algebraic structure
		\item \textbf{Niche Preservation:} Population-based methods should maintain diversity across different equivalence classes
	\end{enumerate}

	\subsubsection{Standard Genetic Algorithm (GA) Approach}
	\textbf{Chromosome Encoding:} Each individual in the GA population represents a potential clinical rule through a chromosome encoding that handles both numeric and categorical clinical metrics:
	\begin{itemize}[noitemsep,topsep=0pt]
		\item \textbf{Numeric Metrics:} Each numeric clinical measurement (e.g., OSDI score, TBUT) is encoded using 3 genes: (1) Binary activation gene (0=unused, 1=used), (2) Binary operator gene (0=$\leq$, 1=$>$), (3) Real-valued threshold gene (within the metric's observed range).
		\item \textbf{Categorical Metrics:} Each categorical measurement (e.g., DED\_TYPE) is encoded using $1+K$ genes: A binary activation gene (0=unused, 1=used), followed by $K$ Binary level selection genes for each of the $K$ categorical levels.
	\end{itemize}
	This encoding naturally maps to our monoid representation through conversion to binary atomic rule vectors, enabling seamless integration with the algebraic framework.
	
	\textbf{Fitness Evaluation:} Each chromosome is decoded into a clinical rule, applied to the patient cohort to generate a subset of patient cohort, and evaluated using the biomarker fold change objective function. The GA optimisation process seeks to maximise this fitness while maintaining subgroup size constraints for statistical validity. Afterall a subcohort of 1 patient makes no sense in practice even with super high fold change.
	
	\subsubsection{Quotient-Space-Aware Genetic Algorithm with Niche Elite Preservation.}
	Building upon the algebraic framework of equivalence class learning, we implement an enhanced GA that incorporates explicit equivalence class detection and niche preservation (where the best individual from each equivalence class is directly carried over to the next generation) to maintain population diversity across functionally distinct rule equivalence classes.
	
	\textbf{Equivalence Class Detection Mechanism:} At regular intervals during evolution (every $k$ generations, typically $k=10$), the algorithm performs equivalence class discovery:
	\begin{enumerate}
		\item Convert chromosomes to binary atomic rule vectors
		\item Cluster via $\epsilon$-functional proximity on objective values: group individuals $x_i, x_j$ if $|f(x_i) - f(x_j)| < \epsilon$ (e.g., DBSCAN on the 1D set of $f$-values)
		\item Preserve niche elites (highest fitness individual) from each equivalence class
	\end{enumerate}
	This ensures population diversity across functionally distinct rule equivalence classes while preventing premature convergence.

	\subsubsection{Algorithm Specification}
	
	The following pseudocode (Algorithm~\ref{alg:ga_orbit_comparison}) presents both standard GA and our quotient-aware extension side-by-side, highlighting the key algorithmic differences:
	
	\begin{algorithm}[H]
		\caption{Genetic Algorithm: Standard vs. Quotient-Aware Comparison}
		\label{alg:ga_orbit_comparison}
		\small
		
		\textbf{Notation:} $P_t$ = population at generation $t$, $\mathbf{v}_i$ = binary atomic rule vector for chromosome $i$, $E$ = elite set (best from each equivalence class), $V$ = set of valid atomic vectors (i.e., $|V|$ = number of individuals (in the sense of genetic algorithm, not patient size)) with non-zero vectors), $\mathcal{C}$ = DBSCAN clusters, $f(x_i)$ = fold change of a specific protein level of the patient subgroup vs healthy volunteers
		
		\vspace{0.3em}
		\textbf{Input:} Population size $N$, generations $T$, crossover/mutation rates $p_c, p_m$ \\
		\textbf{Input (Quotient-aware):} Quotient space check interval $\tau$, DBSCAN radius $\epsilon$, minimum sample size required for each equivalence class ($\text{minPts}$)
		
		\vspace{0.3em}
		\textbf{1.} Initialize $P_0 = \{x_1^{(0)}, \ldots, x_N^{(0)}\}$ randomly; Evaluate $f(x_i^{(0)})$ \hfill [Standard]
		
		\textbf{2.} \textbf{For} $t = 1$ to $T$:
		
		\quad \textbf{2.1} \textbf{Equivalence Class Detection (Quotient-aware only):}
		
		\quad \quad \textbf{If} quotient-aware \textbf{AND} $(t = 1$ \textbf{OR} $t \bmod \tau = 0)$: \hfill [Quotient]
		
		\quad \quad \quad Convert to atomic vectors: $\mathbf{v}_i = \phi(x_i^{(t-1)})$ \hfill [Quotient]
		
		\quad \quad \quad Filter valid vectors: $V = \{\mathbf{v}_i : \mathbf{v}_i \neq \mathbf{0}\}$ \hfill [Quotient]
		
		\quad \quad \quad \textbf{If} $|V| \geq \text{minPts}$: Apply $\mathcal{C} \leftarrow \text{DBSCAN}(V, \epsilon, \text{minPts})$ \hfill [Quotient]
		
		\quad \quad \quad Preserve niche elites from each $k$-th equivalence class \hfill [Quotient]
		\begin{align*}
			E = \bigcup_{C_k \in \mathcal{C}: |C_k| \geq \text{minPts}, k \neq 0} \{\arg\max_{x_i \in C_k} f(x_i)\}
		\end{align*}

		\quad \textbf{2.2} \textbf{Population Construction (Quotient-aware only):}
		
		\quad \quad Initialize: $P_t^{\text{new}} = \text{unique}(E)$ (if quotient-aware and $E$ exists), else $P_t^{\text{new}} = \emptyset$ \hfill [Quotient]
		
		\quad \quad Remaining slots: $n_t = N - |P_t^{\text{new}}|$ \hfill [Quotient]
		
		\quad \textbf{2.3} \textbf{Standard GA Operations:} Generate $n_t$ offspring via selection, crossover, mutation \hfill [Standard]
		
		\quad \textbf{2.4} Update: $P_t \leftarrow P_t^{\text{new}}$; Evaluate $f(x_i^{(t)})$; Track global best \hfill [Standard]
		
		\textbf{3.} \textbf{Return:} Best solution across all generations
	\end{algorithm}
	
	\textbf{Key Algorithmic Differences:}
	\begin{itemize}
		\item \textbf{Standard GA} follows conventional selection-crossover-mutation cycles without diversity preservation mechanisms
		\item \textbf{Quotient-aware GA} introduces periodic detection of functionally similar niches via $\epsilon$-functional proximity on objective values, ensuring niche elite preservation from each discovered equivalence class
		\item The quotient space check interval $\tau$ balances computational overhead with diversity maintenance - frequent checks preserve more diversity but increase runtime
	\end{itemize}

	\subsubsection{Bayesian Optimisation (BO) Approach}
	For comparison, we implemented BO with equivalence class learning integration, though our results demonstrate its limitations for discrete combinatorial spaces:
	\begin{enumerate}
		\item \textbf{Equivalence Class Discovery:} At every BO iteration, cluster observed objective values with DBSCAN \cite{ester1996density} using an $\epsilon$ tolerance to identify empirical equivalence classes (i.e. with $|f_i - f_j| < \epsilon$)
		\item \textbf{Gaussian Process Modeling:} GP surrogate \cite{rasmussen2006gaussian} with Hamming distance kernel: $K(\binaryvector_1, \binaryvector_2) = \theta_0 \exp(-\theta_1 d_H(\binaryvector_1, \binaryvector_2)) + \theta_2 \delta(\binaryvector_1, \binaryvector_2)$, where $\delta$ is kronecker delta
		\item \textbf{Sequential Optimisation:} Expected Improvement acquisition function \cite{jones1998efficient} operating on the quotient space of equivalence class representatives
	\end{enumerate}

	\section{Empirical Evaluation and Results}
	We evaluated five optimisation approaches (GA, quotient-space-aware GA, BO, quotient-space-aware BO, greedy search) across four scenarios: real/synthetic data × with/without numeric features. 
	
	\subsection{Experimental Design}
	\subsubsection{Dataset Configurations}
	Synthetic data was randomly generated so that the variables are identical to the real clinical dataset, but the biomarker values are randomly sampled from uniform distributions for both categorical and continuous variables. 
	
	\subsubsection{Experimental Benchmark Configuration}
	To ensure rigorous and reproducible evaluation, we established a standardised benchmarking protocol across all optimisation methods:
	
	\textbf{Algorithm Parameters:} Genetic algorithms used population sizes of 50-100 individuals with 20-150 generations, single-point crossover (probability 0.8), and mutation rates (0.1). Bayesian optimisation employed Gaussian Process with Hamming distance kernels and Expected Improvement acquisition functions. Detailed information can be found in our github
	
	https://github.com/msunstats/algebraic\_structure\_for\_combinatoric\_problems.
	
	\textbf{Stability Assessment:} To ensure statistical validity of biomarker comparisons, we rerun the same algorithm for 20 times, at minimum subgroup size thresholds of 10, 20, and 30 patients. Within each minimum subgroup size, we randomly choose 5 set of algorithm  parameters as testing all of them is not computationally feasible.

	\subsection{Performance Results}
	Table~\ref{tab:performance_summary} presents performance comparison across all experimental scenarios:
	
	\begin{table}[htbp]
		\centering
		\caption{Performance Summary Across All Experimental Scenarios}
		\label{tab:performance_summary}
		\small
		\begin{tabular}{@{}p{3cm}lrrrr@{}}
			\toprule
			\textbf{Scenario} & \textbf{Method} & \textbf{Runs} & \textbf{Mean} & \textbf{Std} & \textbf{Time} \\
			& & & \textbf{Fitness} & \textbf{Dev} & \textbf{(s)} \\
			\midrule
			Real Data & GA (With Quotient Space) & 1152 & 82.63 & 52.41 & 46.6 \\
			(No Numeric) & GA (No Quotient Space) & 1152 & 79.52 & 50.38 & 20.7 \\
			& BO (With Quotient Space) & 1296 & 53.41 & 44.56 & 13.9 \\
			& BO (No Quotient Space) & 1295 & 53.36 & 44.55 & 14.0 \\
			& Greedy Algorithm & 540 & 58.10 & 46.10 & 0.88 \\
			\midrule
			Real Data & GA (With Quotient Space) & 1152 & 82.74 & 52.94 & 58.9 \\
			(With Numeric) & GA (No Quotient Space) & 1152 & 80.58 & 52.54 & 27.4 \\
			& BO (With Quotient Space) & 1296 & 68.02 & 50.12 & 27.9 \\
			& BO (No Quotient Space) & 1296 & 67.84 & 49.90 & 28.0 \\
			& Greedy Algorithm & 540 & 80.19 & 53.39 & 1.77 \\
			\midrule
			Synthetic Data & GA (With Quotient Space) & 1152 & 3.12 & 0.39 & 41.3 \\
			(No Numeric) & GA (No Quotient Space) & 1152 & 3.08 & 0.40 & 18.0 \\
			& BO (No Quotient Space) & 1296 & 2.73 & 0.49 & 13.7 \\
			& BO (With Quotient Space) & 1296 & 2.72 & 0.49 & 13.6 \\
			& Greedy Algorithm & 540 & 2.85 & 0.51 & 0.84 \\
			\midrule
			Synthetic Data & GA (With Quotient Space) & 1152 & 3.19 & 0.46 & 56.2 \\
			(With Numeric) & GA (No Quotient Space) & 1152 & 3.09 & 0.49 & 23.5 \\
			& BO (No Quotient Space) & 1296 & 2.65 & 0.31 & 27.9 \\
			& BO (With Quotient Space) & 1296 & 2.65 & 0.29 & 27.9 \\
			& Greedy Algorithm & 540 & 2.86 & 0.40 & 1.44 \\
			\bottomrule
		\end{tabular}
	\end{table}
	
	\subsubsection{Performance Relative to Theoretical Global Optimum}
	A critical validation of our optimisation methods comes from comparing their performance to the theoretical global optimum discovered through exhaustive search for discrete optimisation cases. This analysis provides unprecedented insight into how close our heuristic methods come to achieving true optimal solutions. For those scenarios with continuous variables, such theoretical benchmarking is not feasible as the number of combination is simply uncountable. Thus we shall focus on the discrete case only.
	
	Table~\ref{tab:performance_vs_optimum} presents the performance of each method relative to the theoretical global optimum, expressed as the mean/median of the ratio between actual fitness score vs global optimum fitness score, and how frequently each method successfully identifies the exact theoretical global optimum.
	
	\begin{table}[htbp]
		\centering
		\caption{Global Optimum Achievement Analysis by Method}
		\label{tab:performance_vs_optimum}
		\small
		\begin{tabular}{@{}p{3cm}lrrr@{}}
			\toprule
			\textbf{Scenario} & \textbf{Method} & \textbf{Mean} & \textbf{Median} & \textbf{\% of Cases Achieving} \\
			& & \textbf{Ratio} & \textbf{Ratio} & \textbf{Global Optimum} \\
			\midrule
			Real Data & GA (With Quotient Space) & 0.9937 & 1.0000 & 77.26\% \\
			(No Numeric) & GA (No Quotient Space) & 0.9635 & 0.9791 & 34.64\% \\
			& BO (With Quotient Space) & 0.6554 & 0.6253 & 2.55\% \\
			& BO (No Quotient Space) & 0.6574 & 0.6218 & 2.47\% \\
			& Greedy Algorithm & 0.7069 & 0.7095 & 2.78\% \\
			\midrule
			Synthetic Data & GA (With Quotient Space) & 0.9885 & 0.9993 & 48.09\% \\
			(No Numeric) & GA (No Quotient Space) & 0.9760 & 0.9905 & 36.63\% \\
			& BO (No Quotient Space) & 0.8620 & 0.8504 & 1.93\% \\
			& BO (With Quotient Space) & 0.8604 & 0.8504 & 1.93\% \\
			& Greedy Algorithm & 0.8986 & 0.9050 & 2.78\% \\
			\bottomrule
		\end{tabular}
	\end{table}
	
	\textbf{Performance Hierarchy:} GA methods achieve highest performance relative to theoretical optimum, with quotient-space-aware GA consistently outperforming standard GA. BO and Greedy methods show lower performance for this discrete combinatorial problem. Synthetic data results mirror real data trends, but the extent of global optimum discovery of GA is much lower as the synthetic data is generated by sampling from each categorical variable uniformly (whereas in real dry eye patients, some features are indeed enriched vs healthy volunteers), thus the chance of finding a very specific subgroup of patients is lower (in fact the mean fitness score is a lot lower in synthetic data than real data)
	
	\textbf{Quotient Space Learning Advantage:} Quotient-space-aware GA achieves higher global optimum discovery rates (48.09-77.26\%) versus standard GA (34.64-36.63\%) based on categorical variables. BO and Greedy methods rarely find global optima (1.93-2.78\%), demonstrating that exploiting algebraic structure enhances optimisation effectiveness, but it is method dependent. This is likely because the Hamming distance kernel and GP surrogate modelling may be less well-suited to exploiting the discrete quotient space structure compared to population-based evolutionary methods.

	\subsubsection{Impact of Numeric Features on Real Clinical Data Performance}
	A critical aspect of our evaluation examines how the inclusion of numeric clinical features affects optimisation performance. Comparing real clinical data with and without numeric features reveals important insights into the complexity and optimisation landscape of clinical rule discovery.
	
	\textbf{Feature Complexity Analysis:} The inclusion of numeric features substantially increases the search space complexity. From Table~\ref{tab:performance_summary}, we observe that Quotient-space-aware GA maintains consistent performance levels between scenarios (mean fitness 82.63 vs 82.74), suggesting robust handling of increased feature dimensionality. However, Quotient-space-aware BO methods show improved performance with numeric features (68.02 vs 53.41), indicating that BO may benefit from the additional numeric search dimensions.
	
	\textbf{Computational Impact:} The inclusion of numeric features increases computational overhead across all methods. GA execution times increase from 45.8s to 58.9s for quotient-space-aware variants, while BO methods show substantial increases from 13.6s to 27.9s. This overhead reflects the additional complexity in encoding, evaluation, and optimisation of mixed categorical-numeric rule spaces.
	
	\textbf{Quotient Space Learning Effectiveness:} Quotient space learning benefits remain consistent across both feature configurations for GA methods, with similar improvement percentages. This consistency suggests that the algebraic quotient space structure is preserved regardless of feature type composition, validating the robustness of our monoid action framework.
	
	\subsubsection{Method-Specific Quotient Space Learning Analysis}
	Figure~\ref{fig:ga_comparison} and Figure~\ref{fig:bo_comparison} provide detailed individual comparisons of quotient-space-aware versus standard implementations for GA and BO methods respectively, illustrating the contrasting effectiveness of quotient space learning across different algorithmic paradigms.
	
	\begin{figure}[htbp]
		\centering
		\begin{subfigure}[b]{0.48\textwidth}
			\includegraphics[width=\textwidth]{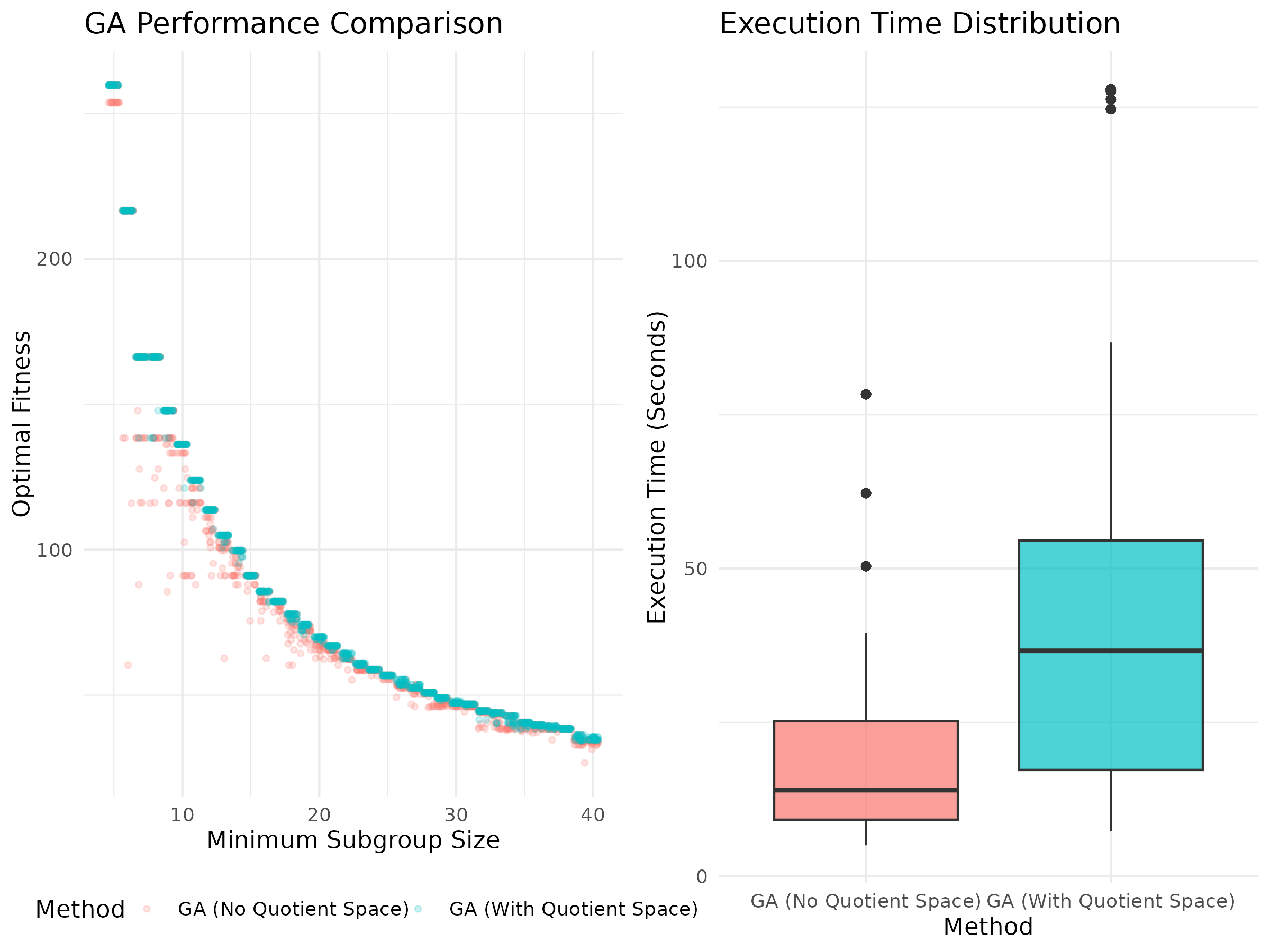}
			\caption{Real Data Without Numeric Features}
		\end{subfigure}
		\hfill
		\begin{subfigure}[b]{0.48\textwidth}
			\includegraphics[width=\textwidth]{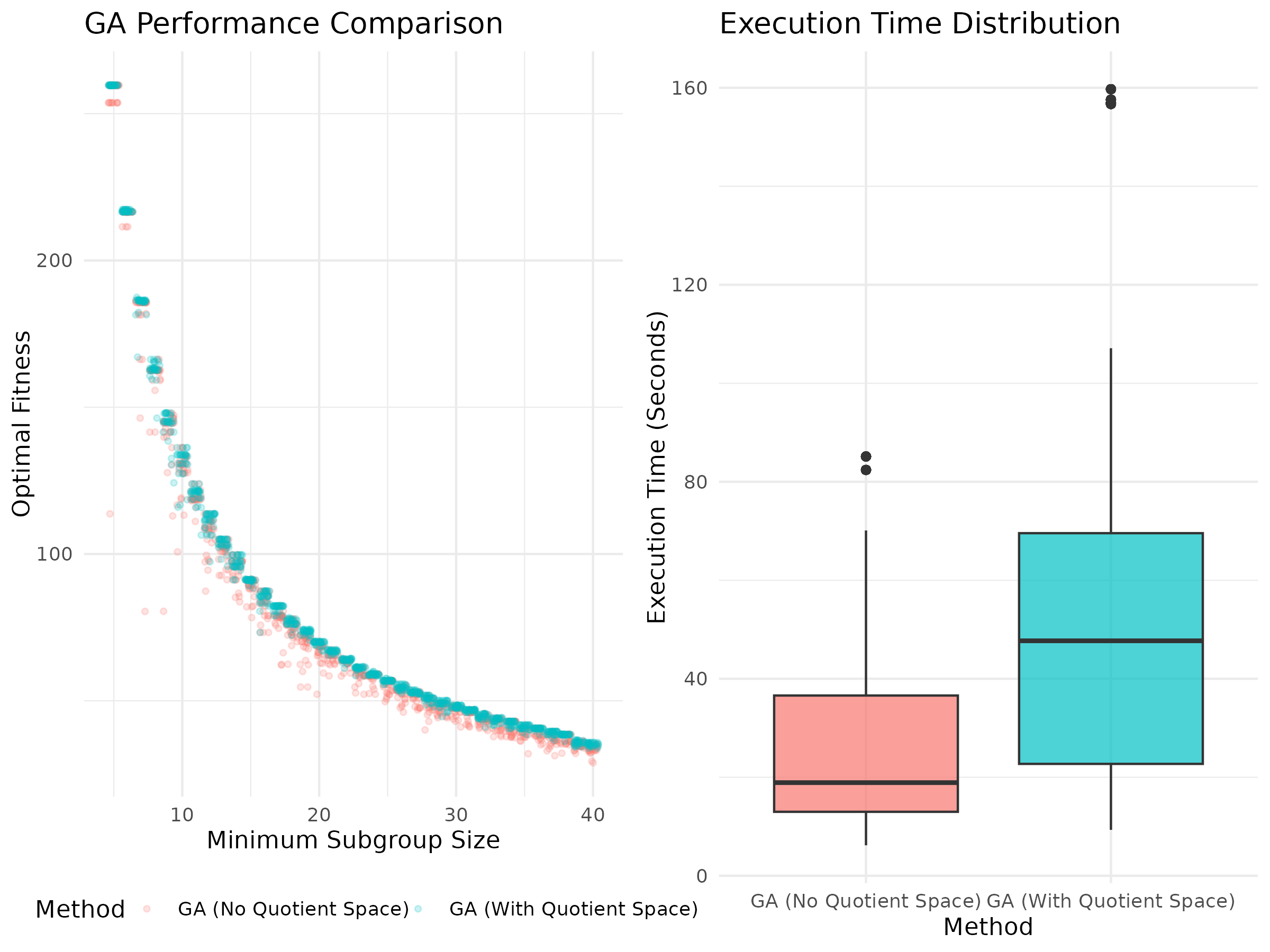}
			\caption{Real Data With Numeric Features}
		\end{subfigure}
		\caption{Genetic Algorithm Performance Comparison: Real Data With vs Without Quotient Space Learning across Feature Configurations. The side-by-side comparison demonstrates consistent improvements when quotient space learning is incorporated into GA methods across both categorical-only (a) and mixed categorical-numeric (b) feature scenarios.}
		\label{fig:ga_comparison}
	\end{figure}
	
	\begin{figure}[htbp]
		\centering
		\begin{subfigure}[b]{0.48\textwidth}
			\includegraphics[width=\textwidth]{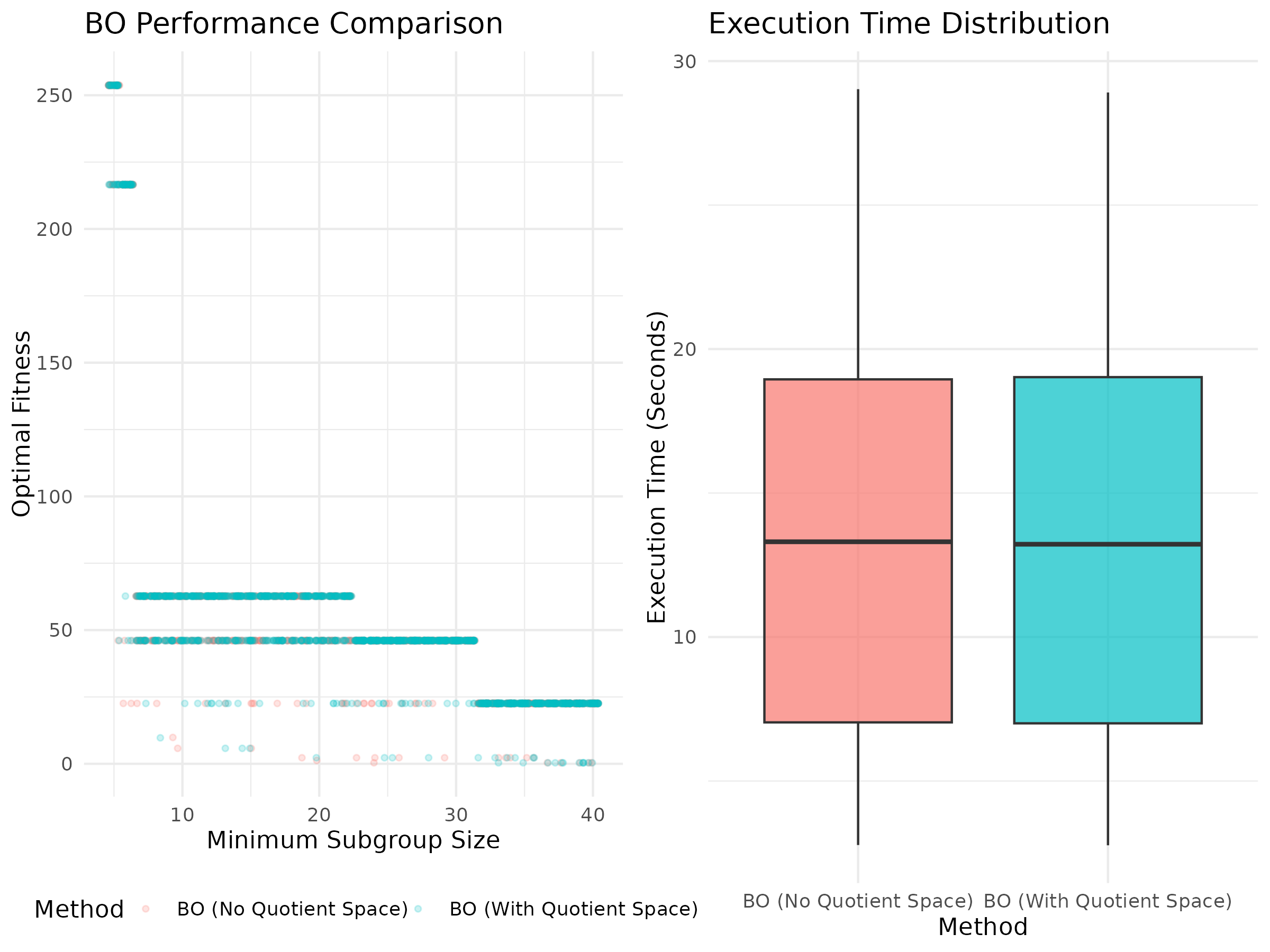}
			\caption{Without Numeric Features}
		\end{subfigure}
		\hfill
		\begin{subfigure}[b]{0.48\textwidth}
			\includegraphics[width=\textwidth]{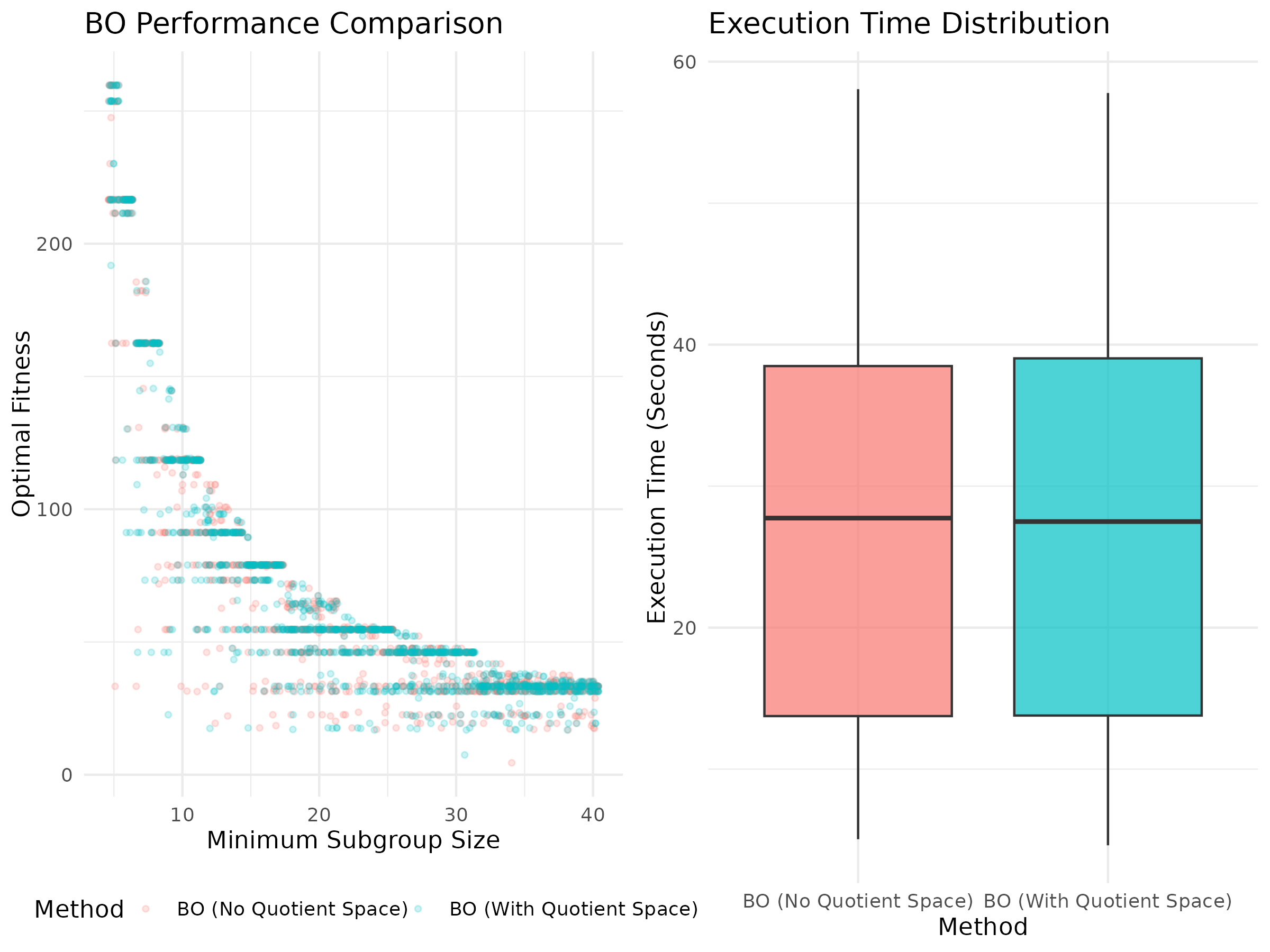}
			\caption{With Numeric Features}
		\end{subfigure}
		\caption{Bayesian Optimisation Performance Comparison: With vs Without Quotient Space Learning across Feature Configurations. In contrast to GA methods, BO approaches show mixed and inconsistent results from quotient space learning integration across both feature scenarios, suggesting that the Hamming distance kernel and quotient space approach may be less optimal for BO-based discrete optimisation in this domain.}
		\label{fig:bo_comparison}
	\end{figure}
	
	\subsection{Stability and Robustness Analysis}
	Figures~\ref{fig:stability_analysis_no_numeric} and~\ref{fig:stability_analysis_with_numeric} demonstrate that quotient-space-aware methods achieve superior performance with enhanced stability and robustness across both categorical-only and mixed feature scenarios.
	
	\begin{figure}[htbp]
		\centering
		\begin{subfigure}[b]{\textwidth}
			\includegraphics[width=\textwidth]{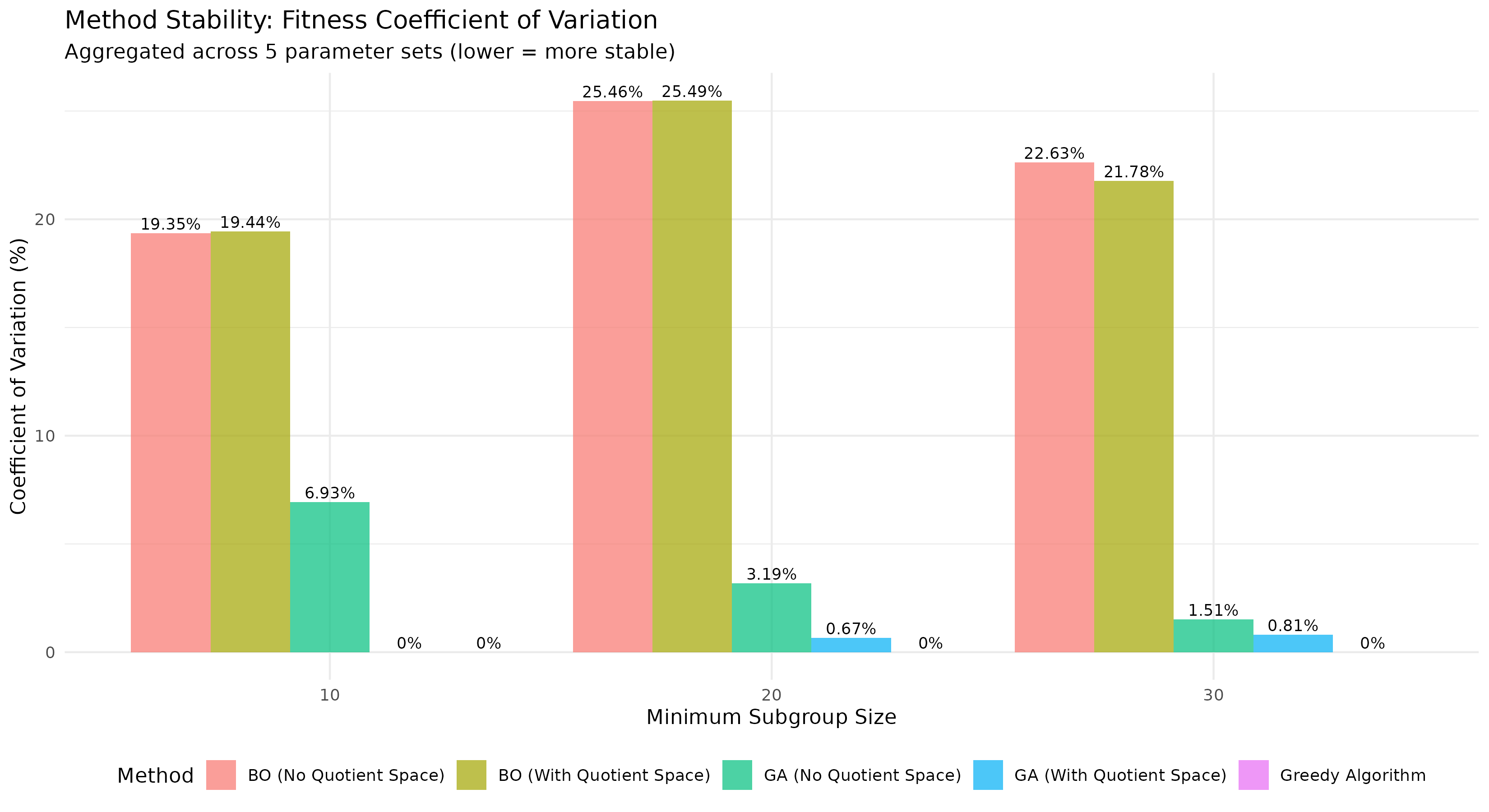}
			\caption{Coefficient of Variation}
		\end{subfigure}
		\hfill
		\begin{subfigure}[b]{\textwidth}
			\includegraphics[width=\textwidth]{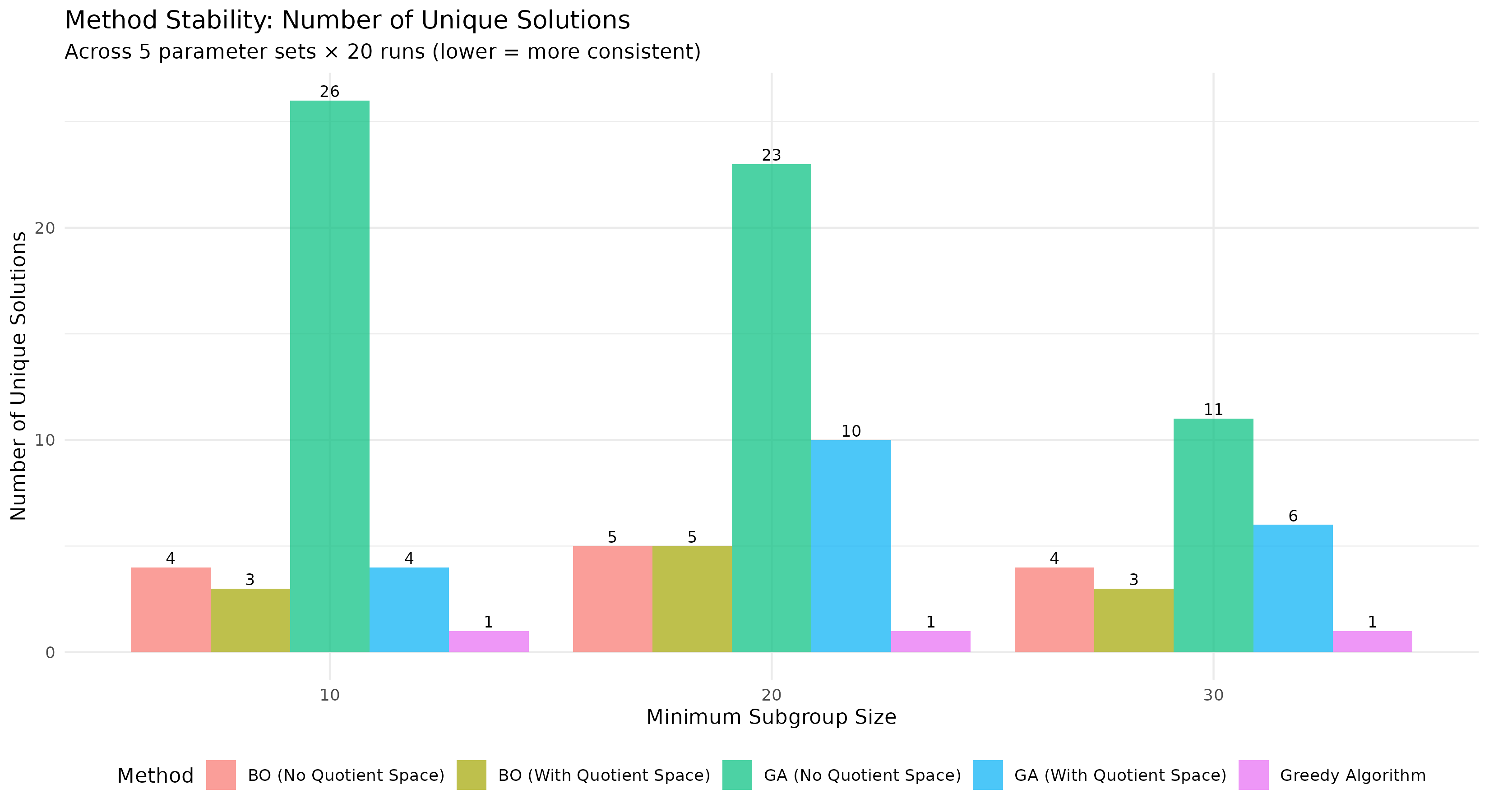}
			\caption{Solution Diversity}
		\end{subfigure}
		\caption{Comprehensive Stability Analysis for Real Clinical Data without Numeric Features. Solution diversity metrics reveal that quotient space learning maintains exploration capabilities while improving convergence quality even with simplified feature spaces.}
		\label{fig:stability_analysis_no_numeric}
	\end{figure}
	
	\begin{figure}[htbp]
		\centering
		\begin{subfigure}[b]{\textwidth}
			\includegraphics[width=\textwidth]{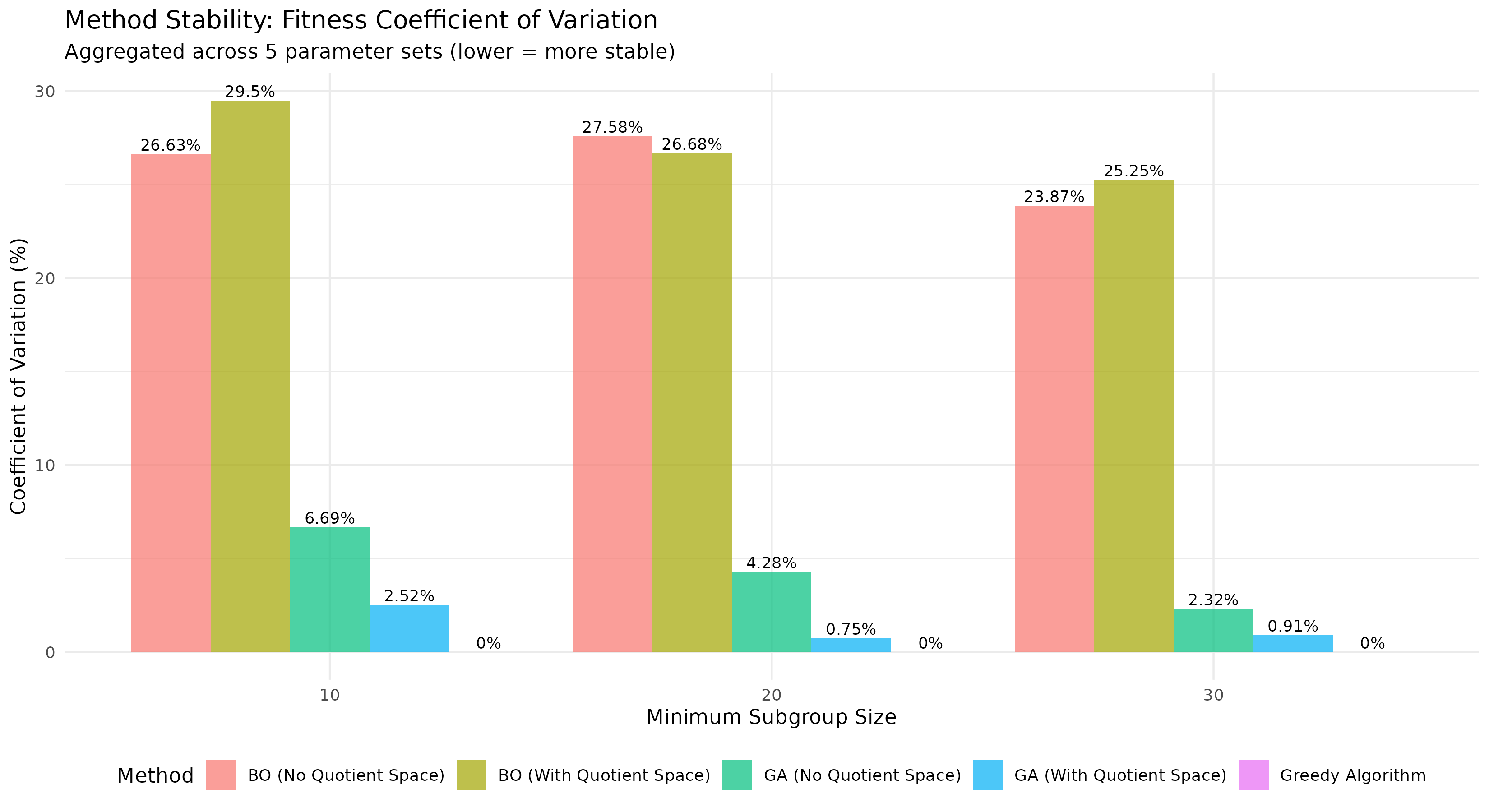}
			\caption{Coefficient of Variation}
		\end{subfigure}
		\hfill
		\begin{subfigure}[b]{\textwidth}
			\includegraphics[width=\textwidth]{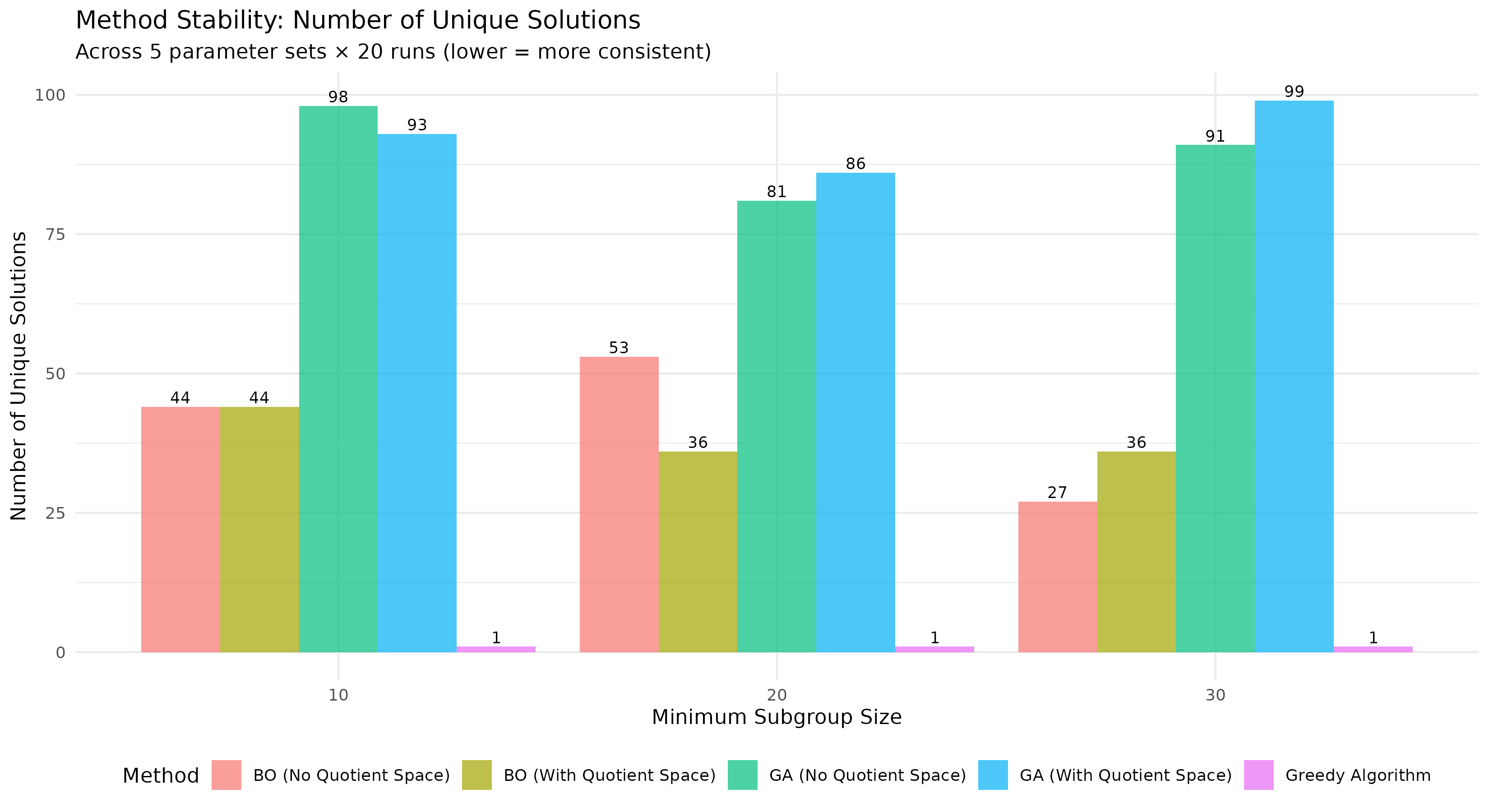}
			\caption{Solution Diversity}
		\end{subfigure}
		\caption{Comprehensive Stability Analysis for Real Clinical Data with Numeric Features. (a) Coefficient of variation analysis demonstrates reduced variability in quotient-space-aware GA, indicating enhanced algorithmic stability in mixed feature spaces. (b) Number of unique solutions with continuous data is not appropriate as continuous threshold encoding can generate mathematically distinct but clinically equivalent rules (e.g., x > 12.5 vs x > 12.501), artificially inflating diversity measures without corresponding clinical significance.}
		\label{fig:stability_analysis_with_numeric}
	\end{figure}
	
	\subsection{Clinical Implications and Translational Relevance}
	The superior performance of quotient-space-aware GA methods has direct clinical implications: enhanced biomarker discovery through identifying optimal patient subgroups, robust rule discovery validated against theoretical optimum, and computational feasibility for clinical research applications with reasonable execution times (in our case under 1 minute).
	
	This comprehensive case study demonstrates the practical effectiveness of our algebraic framework. We now examine the general applicability conditions and explore how this framework could benefit other combinatorial optimisation domains.

	\section{General Framework: Applicability and Extensions}
	
	The patient stratification case study demonstrates how algebraic structure can be systematically exploited for combinatorial optimisation. This suggests a broader framework applicable to many combinatorial optimisation problems with similar structural properties. 
	
	\subsection{Framework Applicability Conditions}
	
	Our algebraic structure discovery framework applies to combinatorial optimisation problems satisfying these conditions:
	
	\begin{enumerate}
		\item \textbf{Compositional Structure:} The problem involves combining discrete components (atomic rules, filters, constraints) through well-defined operations.
		\item \textbf{Associative Operations:} The combination operation exhibits associativity, enabling formation of algebraic structures (monoids, semigroups, groups).
		\item \textbf{Functional Equivalence:} Multiple distinct combinations yield similar outcomes, creating natural equivalence classes.
		\item \textbf{Expensive Evaluation:} Objective function evaluation is computationally expensive or the search space is simply too large to be explored combinatorically.
	\end{enumerate}
	
	\subsection{Potential Applications Beyond Current Case Studies}
	
	This framework could potentially benefit other combinatorial optimisation domains where logical rule combination creates algebraic structure.

	\subsubsection{Rule-Based Molecular Screening}
	
	In modern drug discovery, virtual libraries may contain millions to billions of candidate molecules. 
	It is computationally infeasible to evaluate every compound with expensive methods such as docking, molecular dynamics, or quantum chemistry. 
	Instead, industry-standard practice applies inexpensive, interpretable \emph{rule-based filters}, such as Lipinski’s Rule of Five, Veber’s rules, PAINS filters, and REOS, to remove compounds with poor druglikeness or undesirable substructures, adhering to the 'fail early, fail cheaply' principle \cite{Oprea2002,VSbestPractice2021,PhamThe2020,Sukhachev2025}. 
	After this, progressively more accurate and expensive \emph{in silico} predictions are used to prioritise candidates.
	
	We propose a principled algebraic framework for \emph{discovering optimal rule sets} tailored to a given objective (e.g., maximising predicted binding affinity or optimising ADME properties). The discovered rules can be used to guide future optimisation of molecular design, and the yielding subset of molecule candidates can be moved forward to synthesis and downstream wetlab testing. For this case study only a sketch of the full framework is provided.
	
	The framework relies on three key insights:
	
	\begin{enumerate}
		\item \textbf{Monoid structure of rules.}  
		Each filter $a_i$ is a Boolean predicate on molecules. 
		Combining filters by conjunction ($a_i \land a_j$) is associative, admits an identity element (the empty rule), and thus forms an abelian monoid $(\mathcal{S}, \land, \varepsilon)$.
		
		\item \textbf{Action on molecules.}  
		The monoid $\mathcal{S}$ acts on the compound library $C_0$. 
		A screening strategy $s \in \mathcal{S}$ yields the subset 
		\[
		C_s = \{ c \in C_0 \mid c \text{ satisfies all filters in } s \}.
		\]
		
		\item \textbf{Quotient space of filter sets.}  
		Different sequences of the same filters (e.g., $(a_1,a_2,a_3)$ vs.\ $(a_3,a_1,a_2)$) yield the same surviving set $C_s$. 
		We define an equivalence relation $s_1 \sim s_2$ if $C_{s_1} = C_{s_2}$. 
		The resulting quotient space $\mathcal{S}/\!\sim$ consists of unique \emph{sets} of filters rather than their permutations. 
		This reduces the search space from factorial size ($O(N!)$ ordered sequences) to exponential size ($O(2^N)$ unordered subsets). When the number of criteria is large, factorial is way larger than exponential in astronomic scale.
	\end{enumerate}
	
	\noindent
	\textbf{Optimisation strategy.} 
	We separate the discovery and cost objectives:
	
	\begin{itemize}
		\item \emph{Step 1 (Discovery).} Search the quotient space $\mathcal{S}/\!\sim$ using a genetic algorithm or other black-box optimiser to identify the optimal filter set $s^{\ast}$ that maximises efficacy (e.g., predicted affinity of $C_{s^{\ast}}$).
		
		\item \emph{Step 2 (Cost optimisation).} 
		Once the best filter set $s^{\ast}$ has been identified, the final molecules $C_{s^{\ast}}$ are independent of the order in which the filters are applied. 
		However, the \emph{execution cost} does depend on order: applying a cheap and selective filter early can reduce the number of molecules passed to later, more expensive filters.
		
		Let each filter $a_i \in s^{\ast}$ have cost $c_i$ (per molecule) and selectivity $s_i \in [0,1]$ (fraction removed). 
		If filter $a_j$ is applied after $a_i$, its expected cost is $(1-s_i)c_j$, since only a fraction $(1-s_i)$ of molecules survive from the previous stage. 
		For a sequence of filters $\pi=(i_1,\dots,i_m)$, the expected per-molecule cost is therefore
		
		\begin{equation}
			\mathbb{E}[\text{Cost}(\pi)] = c_{i_1} + (1-s_{i_1})c_{i_2} + (1-s_{i_1})(1-s_{i_2})c_{i_3} + \cdots \nonumber
		\end{equation}
		
		Thus intuitively we should place $i$ before $j$ if $c_i/s_i \leq c_j/s_j$ (assume $s_i, s_j > 0$), which conceptually implies that the cheapest and most selective filters should be applied first.
	\end{itemize}
	
	\noindent
	\textbf{Discussion: numeric vs.\ categorical criteria.} 
	Many molecular properties used in screening are inherently \emph{numeric} (e.g., $\log P$, molecular weight, polar surface area, predicted solubility, docking score). 
	For such continuous properties, one could train a neural network (NN) to predict ADME outcomes and then apply reverse optimisation (e.g., gradient-based search) to directly identify molecules with desirable property profiles. 
	
	However there also exists categorical/boolean filters, e.g. PAINS motif, satisfies Lipinski's Rule of Five, absence of toxic substructures.	These are not easily handled by smooth optimisation and are naturally represented in our algebraic framework as atomic predicates combined via conjunction. The advantage is that the resulting rules are \emph{interpretable}: they provide chemists with concrete design guidelines (e.g., 'molecular weight $< 450$ Daltons, $2 < \log P < 4$, no PAINS motif'). 
	
	Thus, our framework is especially suited for the 'cheap, interpretable filter' stage of drug discovery funnels, while NN-based optimisation can operate downstream on the survivors. 
	
	\noindent
	\textbf{Practical benefit.}  
	This framework transforms molecular screening into a two-stage optimisation: (i) discover the \emph{best region of chemical space} defined by simple, interpretable rules; (ii) execute the chosen strategy with minimal computational expense. 
	The discovered rule sets are both actionable (directly yielding candidate subsets for synthesis and wet-lab testing) and interpretable (highlighting physicochemical criteria that can guide future optimisation campaigns).

	\subsubsection{Feature Selection and Higher-Order Synergies}
	
	\textbf{Problem Setting.} In high-dimensional machine learning, selecting informative feature subsets from $n$ variables is essential. 
	The search space is the power set of features, of size $2^n$, since each feature can be either included or excluded. 
	Exhaustive search is intractable, while standard heuristics (stepwise selection, LASSO) often fail to capture feature interactions.
	
	\textbf{Higher-Order Synergies.} 
	Even in additive models, feature utility can depend on context:
	\begin{itemize}[noitemsep]
		\item \emph{Cancellation:} if $y = x_1 - x_2$ with $x_1 \approx x_2$, then $x_1$ or $x_2$ alone appears uninformative, but the pair $\{x_1,x_2\}$ explains variance well.
		\item \emph{Nonlinear interaction:} if $y = x_1x_2$, then neither $x_1$ nor $x_2$ has predictive power individually, but the interaction is perfectly predictive.
	\end{itemize}
	Thus, relevant information can reside in subsets or constructed interactions rather than single features.
	
	\textbf{Framework Application.} 
	Each atomic decision 'include feature $i$' is a generator; subsets of features form a commutative monoid under union, with the empty set as identity. 
	Equivalence classes are defined empirically: two subsets are equivalent if they yield (approximately) the same predictive performance. 
	Quotient-space exploration focuses search on \emph{classes of feature sets} rather than individual subsets, reducing redundancy and highlighting higher-order synergies.

	\subsubsection{Molecular Design under Symmetry}
	
	\textbf{Problem Setting.} 
	In computational chemistry and molecular design, candidate molecules are often represented as graphs (atoms and bonds) or as three-dimensional coordinate structures. The search space of possible molecules or conformations is vast, and many representations are redundant because they describe the same molecule under a symmetry transformation (e.g., rotations, reflections, or permutations of identical atoms). 
	
	\textbf{Pain Points.} 
	Na\"{i}ve enumeration of molecules treats each labeling or orientation as distinct, even when they are chemically equivalent. This creates enormous redundancy. Similarly, in molecular docking, ligands with internal symmetry (such as benzene or symmetric peptides) often produce multiple binding poses that are simple rotations or reflections of one another. Docking algorithms that are unaware of symmetry report these as distinct solutions, wasting computation with redundant binding modes.
	
	\textbf{Framework Application.} 
	Symmetry operations form a group, acting on molecular coordinates or adjacency matrices. Two molecular representations are equivalent if they belong to the same orbit under this group action. The quotient space of molecular representations modulo symmetry collapses redundant labelings and orientations into canonical equivalence classes. Optimisation or enumeration can then proceed in this quotient space, ensuring that each chemically distinct structure, binding pose, or conformation is considered exactly once.

	\section{Conclusions, Limitations, and Future Work}
	
	We have introduced a systematic framework for discovering and exploiting algebraic structure in combinatorial optimisation problems. Through case studies in patient stratification and molecular screening, we demonstrated that:
	
	\begin{enumerate}
		\item Many real-world combinatorial problems exhibit hidden algebraic structure amenable to formal mathematical analysis
		\item Abstract algebra provides powerful tools for identifying and characterising this structure
		\item Quotient space construction enables dramatic search space reduction while preserving optimisation objectives
		\item Structure-aware algorithms can achieve substantial performance improvements over structure-agnostic approaches
	\end{enumerate}
	
	Our empirical patient subgroup finding case study shows that quotient-space-aware genetic algorithms achieve global optimum in 48-77\% of cases versus 35-37\% for standard approaches, demonstrating the practical value of algebraic structure exploitation.
	
	\subsection{Methodological Limitations and Considerations}
	
	While the proposed framework offers a novel and efficient approach to optimising over semigroup-structured combinatorial spaces, several methodological limitations must be acknowledged:
	
	\subsubsection{Properties of Empirical Equivalence Class Discovery}
	The empirical discovery of equivalence classes via clustering is a cornerstone of our framework, enabling significant search space reduction. However, its effectiveness relies on several factors:
	
	\begin{enumerate}
		\item \textbf{Definition of $\epsilon$-Equivalence:} The threshold $\epsilon$ in Remark~\ref{rem:practical-equivalence} directly dictates equivalence class granularity, acting as a regularisation parameter that could be tuned through cross-validation or informed by domain expertise regarding clinically meaningful differences.
		\item \textbf{Clustering Robustness:} The reliability of identified equivalence classes depends on clustering algorithm robustness and objective function noise characteristics.
		\item \textbf{Representative Selection:} Our strategy of selecting the highest-performing rule from each cluster may not guarantee true optimality for the entire equivalence class.
	\end{enumerate}
	
	Additionally, not all combinatorial problems exhibit useful algebraic structure. The framework's effectiveness depends on the specific mathematical properties of the problem domain. Problems lacking associative combination operations or meaningful equivalence classes may not benefit from this approach. Additionally, the framework is most suitable for problems where solutions can be naturally decomposed into conjunctive combinations of atomic components.
	
	\subsection{Future Research Directions}
	
	Several promising avenues for future research emerge from this work:
	
	\begin{itemize}
		\item \textbf{Extension to Complex Logical Forms:} Generalisation beyond conjunctive rules to more complex logical forms (e.g., AND of ORs, OR of ANDs), requiring extended algebraic structures and specialised kernels for Gaussian Processes.
		\item \textbf{Advanced Algebraic Structures:} Extension to more complex algebraic structures (e.g. groups or rings) beyond monoids and semigroups.
		\item \textbf{Automatic Structure Detection:} Development of algorithms for automatic detection of algebraic structure in arbitrary combinatorial problems.
		\item \textbf{Theoretical Analysis:} Rigorous theoretical analysis of quotient space optimisation guarantees and convergence properties.
		\item \textbf{Industrial Applications:} Application to large-scale industrial optimisation challenges across diverse domains.
		\item \textbf{ML Integration:} Integration with modern machine learning and artificial intelligence methods, particularly deep learning approaches where  equivalence classes are incorporated in the design of neural network structures.
		\item \textbf{Enhanced discovery of the topology of equivalence classes:} Exploration of alternative methods to better detect the topology of equivalence class discovery, enhancing stability and accuracy.
		\item \textbf{Sophisticated Representative Selection:} Develop more sophisticated methods to better select representatives from each equivalence class.
		\item \textbf{Application of Reinforcement Learning:} Since playing Super Mario can be learnt by reinforcement learning, and Super Mario is in some sense very similar to patient stratification problem. Naturally we may explore reinforcement learning to the same problem.
	\end{itemize}
	
	This work establishes abstract algebra as a valuable lens for understanding and solving complex combinatorial optimisation problems, opening new avenues for research at the intersection of pure mathematics and practical optimisation.
	
	\section{Code and Data Availability}
	All source code, experimental benchmarks, and statistical analyses are available under open-source license at https://github.com/msunstats/algebraic\_structure\_for\_combinatoric\_problems. The repository includes complete R implementations, synthetic data generators, experimental configurations, and statistical analysis scripts. Clinical datasets cannot be shared due to data confidentiality restrictions, but synthetic generators provide equivalent grounds for benchmarking.

	\bibliographystyle{plain} 
	\bibliography{references} 
	
\end{document}